\documentclass[12pt,twoside]{article}
\usepackage{latexsym,amsmath,amssymb}
\usepackage{wrapfig}
\usepackage{tikz}

\newcommand{\eqn}[1]{\begin{align}#1\end{align}}
\newcommand{\eq}[1]{\begin{align*}#1\end{align*}}
\newcommand{\proofof}{Proof of }
\newcommand{\proofend}{}

\def\subsubsect#1{\vspace{1ex plus 0.5ex minus 0.5ex}\noindent{\bf\boldmath{#1.}}}

\usetikzlibrary{trees,arrows,automata}
\usetikzlibrary{decorations.pathmorphing}
\usetikzlibrary{decorations.markings}
\renewcommand{\v}[1]{{\boldsymbol #1}}
\newcommand{\argmax}{\operatornamewithlimits{arg\,max}}
\newcommand{\ind}[1]{[\![ #1 ]\!]}
\newcommand{\R}[0]{\mathbb R}
\newcommand{\N}[0]{\mathbb N}

\usepackage{amsthm}
\theoremstyle{plain}
\newtheorem{theorem}{Theorem}
\newtheorem{lemma}[theorem]{Lemma}
\newtheorem{proposition}[theorem]{Proposition}
\theoremstyle{definition}
\newtheorem{definition}[theorem]{Definition}
\newtheorem{assumption}[theorem]{Assumption}
\newtheorem{example}[theorem]{Example}
\theoremstyle{remark}

\newenvironment{keywords}{\centerline{\bf\small
Keywords}\begin{quote}\small}{\par\end{quote}\vskip 1ex}

\newcommand{\tpic}[1]{
\tikzstyle{state} = [circle,draw,minimum width=0.5cm, minimum height=0.5cm]
\begin{tikzpicture}[->,>=stealth',shorten >=1pt,auto,node distance=\nodedist, semithick]
\scriptsize
{
#1
}
\end{tikzpicture}
}

\renewcommand{\H}[0]{\mathcal H}

\renewcommand{\O}[0]{\mathcal O}

\newcommand{\A}[0]{\mathcal A}
\renewcommand{\Re}[0]{\mathcal R}
\newcommand{\Rw}[0]{{\boldsymbol{R}}}
\renewcommand{\S}[0]{\ensuremath{\mathcal S}}

\renewcommand{\d}[0]{{\mathbf d}}
\newcommand{\du}[1]{{\boldsymbol{d}^{#1}}}
\newcommand{\dt}[2]{{d_{#2}^{#1}}}

\newcommand{\tdu}[1]{{\boldsymbol{\tilde d}^{#1}}}

\newcommand{\D}[2]{{D_{#1,#2}}}

\newcommand{\nodedist}[0]{1.2cm}
\newcommand{\shortnodedist}[0]{0.70cm}
\newcommand{\start}{{\tiny \bfseries $\mathcal S$}}

\begin{document}
%%%%%%%%%%%%%%%%%%%%%%%%%%%%%%%%%%%%%%%%%%%%%%%%%%%%%%%%%%%%%%%
%%                    Title - Page                           %%
%%%%%%%%%%%%%%%%%%%%%%%%%%%%%%%%%%%%%%%%%%%%%%%%%%%%%%%%%%%%%%%

\author{Tor Lattimore \and Marcus Hutter}

\title{%\vspace{-4ex}
\vskip 2mm\bf\Large\hrule height5pt \vskip 4mm
Time Consistent Discounting
\vskip 4mm \hrule height2pt}
\author{{\bf Tor Lattimore$^1$} and {\bf Marcus Hutter$^{1,2}$}\\[3mm]
\normalsize Research School of Computer Science \\[-0.5ex] %, NICTA \\
\normalsize $^1$Australian National University \hspace{1cm} $^2$ETH Z{\"u}rich \\[-0.5ex]
\normalsize\texttt{\{tor.lattimore,marcus.hutter\}@anu.edu.au}
}
\date{15 July 2011}
%\vspace*{-5ex}

\maketitle

\begin{abstract}
A possibly immortal agent tries to maximise its summed discounted rewards over time, where discounting is used to avoid infinite
utilities and encourage the agent to value current rewards more than future ones. Some commonly used discount functions lead
to time-inconsistent behavior where the agent changes its plan over time. These inconsistencies can lead to very poor behavior.
We generalise the usual discounted utility model to one where the discount function changes with the age of the agent. We then
give a simple characterisation of time-(in)consistent discount functions and show the existence of a rational policy for an
agent that knows its discount function is time-inconsistent.
\def\contentsname{\centering\normalsize Contents}
{\parskip=-2.7ex\tableofcontents}
\end{abstract}

\newpage
%%%%%%%%%%%%%%%%%%%%%%%%%%%%%%%%%%%%%%%%%%%%%%%%%%%%%%%%%%%%%%%
\section{Introduction}
%%%%%%%%%%%%%%%%%%%%%%%%%%%%%%%%%%%%%%%%%%%%%%%%%%%%%%%%%%%%%%%

The goal of an agent is to maximise its expected utility; but how do we measure utility? One method is to
assign an instantaneous reward to particular events, such as having a good meal, or a pleasant walk. It would be
natural to measure the utility of a plan (policy) by simply summing the expected instantaneous rewards,
but for immortal agents this may lead to infinite utility and also assumes
rewards are equally valuable irrespective of the time at which they are received.

One solution, the discounted utility (DU) model introduced by
Samuelson in \cite{Sam37}, is to take a weighted sum of the rewards with earlier rewards usually valued more than
later ones.

There have been a number of criticisms of the DU model, which we will not discuss. For an excellent
summary, see \cite{FOO02}. Despite the criticisms, the DU model is widely used in both economics and computer science.

A discount function is time-inconsistent if plans chosen to maximise expected discounted utility change over time. For
example, many people express a preference for \$110 in 31 days over \$100 in 30 days, but reverse that preference 30 days later
when given a choice between \$110 tomorrow or \$100 today \cite{FGM94}. This behavior can be caused by a rational agent with
a time-inconsistent discount function.

Unfortunately, time-inconsistent discount functions can lead to extremely bad behavior and so it becomes important to
ask what discount functions are time-inconsistent.

Previous work has focussed on a continuous model where agents can take actions at any time in a continuous time-space.
We consider a discrete model where agents act in finite time-steps. In general this is not a limitation since any
continuous environment can be approximated arbitrarily well by a discrete one. The discrete setting has the
advantage of easier analysis, which allows us
to consider a very general setup where environments are arbitrary finite or infinite Markov decision processes.

Traditionally, the DU model has assumed a sliding discount function. Formally, a sequence of instantaneous
utilities (rewards)
$R = (r_k, r_{k+1}, r_{k+2}, \cdots)$ starting at time $k$, is given utility equal to $\sum_{t=k}^\infty \dt{}{t-k} r_{t}$ where
$\du{} \in  [0, 1]^\infty$.
We generalise this model as in \cite{Hut06} by allowing the discount function to depend on the age of the agent.
The new utility is given by
$\sum_{t=k}^\infty \dt{k}{t} r_{t}$. This generalisation is consistent with how some agents tend to behave; for example, humans  becoming temporally less myopic
as they grow older.

Strotz \cite{Str55} showed that the only time-consistent sliding discount function is geometric discounting. We
extend this result to a full characterisation of time-consistent discount functions where the discount function is
permitted to change over time. We also show that discounting functions that are ``nearly'' time-consistent give rise to low
regret in the anticipated future changes of the policy over time.

Another important question is what policy should be adopted by an agent that knows it is time-inconsistent. For example, if it
knows it will become temporarily myopic in the near future then it may benefit from paying a price to pre-commit
to following a particular policy. A number of authors
have examined this question in special continuous cases, including \cite{Gol80,BM73,Pol68,Str55}. We
modify their results to our general, but discrete, setting using game theory.

The paper is structured as follows. First the required notation is introduced (Section 2). Example discount functions and the consequences
of time-inconsistent discount functions are then presented (Section 3).
We next state and prove the main theorems, the complete classification of discount functions and the continuity result (Section 4). The game theoretic view
of what an agent {\it should} do if it knows its discount function is changing is analyzed (Section 5). Finally we offer some discussion and
concluding remarks (Section 6).

%%%%%%%%%%%%%%%%%%%%%%%%%%%%%%%%%%%%%%%%%%%%%%%%%%%%%%%%%%%%%%%
\section{Notation and Problem Setup}
%%%%%%%%%%%%%%%%%%%%%%%%%%%%%%%%%%%%%%%%%%%%%%%%%%%%%%%%%%%%%%%

The general reinforcement learning (RL) setup involves an agent interacting sequentially with an environment where in each
time-step $t$ the agent chooses some action $a_t \in \A$, whereupon it receives a reward
$r_t \in \Re \subseteq \R$ and observation $o_t \in \O$.
The environment can be formally defined
as a probability distribution $\mu$ where
$\mu(r_t o_t | a_1 r_1 o_1 a_2 r_2 o_2 \cdots a_{t - 1} r_{t - 1} o_{t - 1} a_t)$ is the probability of
receiving reward $r_t$ and observation $o_t$ having taken action $a_t$ after history $h_{<t} := a_1 r_1 o_1 \cdots a_{t - 1} r_{t - 1} o_{t - 1}$. For convenience, we assume that for a given
history $h_{<t}$ and action $a_t$, that $r_t$ is fixed (not stochastic). We denote the set of all finite histories $\H := (\A\times\Re\times\O)^*$ and
write $h_{1:t}$ to be a history of length $t$, $h_{<t}$ to be a history of length $t-1$. $a_k$, $r_k$, and $o_k$ are the $k$th action/reward/observation
tuple of history $h$ and will be used without explicitly redefining them (there will always be only one history ``in context'').

\setlength{\intextsep}{0pt}
\begin{wrapfigure}[6]{r}{4.5cm}
\topsep=0.0cm
\vspace{-0.1cm}
\tpic{
[->,>=stealth',shorten >=1pt,auto,node distance=5.4cm, semithick]
\tikzstyle{every state}=[fill=none,draw=black,text=black, node distance=2.0cm]
\node[state] (a) {\start};
\node[state] (b) [above right of=a] {};
\node[state] (c) [below right of=a] {};
\node[state] (d) [right of=b] {};
\node[state] (e) [right of=c] {};
\node[state] (f) [right of=d] {};
\node[state] (g) [right of=e] {};
\node[state, draw=none] (h) [right of=f] {};
\node[state, draw=none] (i) [right of=g] {};

\path (a) edge node {$1.0$} (b)
      (a) edge node [left] {$0.8$} (c)
    (b) edge node {$0.7$} (d)
    (b) edge [bend left=30] node [right] {$0.8$} (e)
      (c) edge [bend left=30] node [left] {$1.0$} (d)
      (c) edge node {$0.5$} (e)
      (d) edge node {$0.7$} (f)
      (d) edge [bend left=30] node  [right] {$0.8$} (g)
      (e) edge [bend left=30] node  [right] {$1.0$} (f)
      (e) edge node {$0.5$} (g)
      (f) edge [dashed] node {} (h)
      (g) edge [dashed] node {} (i)
      ;
}
\end{wrapfigure}
A deterministic environment (where every value of $\mu(\cdot)$ is either 1 or 0) can be represented as a graph with edges for actions, rewards
of each action attached to the corresponding edge, and observations in the nodes. For example, the deterministic environment on the right represents an
environment where either pizza or pasta must be chosen at each time-step (evening). An action leading to an upper node is {\tt eat pizza}
while the ones leading to a lower node are
{\tt eat pasta}. The rewards are for a consumer who prefers
pizza to pasta, but dislikes having the same food twice in a row. The starting node is marked as $\mathcal S$. This example, along with all those
for the remainder of this paper, does not require observations.

The following assumption is required for clean results, but may be relaxed if an $\epsilon$ of slop is permitted in some results.
\begin{assumption}\label{assumption}
We assume that $\A$ and $\O$ are finite and that $\Re = [0, 1]$.
\end{assumption}

\begin{definition}[Policy]
A {\it policy} is a mapping $\pi:\H\to\A$ giving an action for each history.
\end{definition}
Given policy $\pi$ and history $h_{1:t}$ and $s \leq t$ then the probability of reaching history $h_{1:t}$ when starting from history
$h_{<s}$ is $P(h_{s:t} | h_{<s}, \pi)$ which is defined by,
\eqn{
\label{eqn-trans-prob}  P(h_{s:t} | h_{<s}, \pi) := \prod_{k=s}^{t} \mu(r_k o_k | h_{<k} \pi(h_{<k})).
}
If $s = 1$ then we abbreviate and write $P(h_{1:t} | \pi) :=  P(h_{1:t} | h_{<1}, \pi)$.
\begin{definition}[Expected Rewards]\label{defn_rewards}
When applying policy $\pi$ starting from history $h_{<t}$, the expected sequence of rewards $\Rw^\pi(h_{<t}) \in [0,1]^\infty$, is defined by
\eq{
R^\pi(h_{<t})_k := \sum_{h_{t:k}} P(h_{t:k} | h_{<t}, \pi) r_k.
}
If $k < t$ then $R^\pi(h_{<t})_k := 0$.
\end{definition}
Note while the set of all possible $h_{t:k} \in (\A\times\Re\times\O)^{k-t+1}$ is uncountable due
to the reward term, we sum only over the possible rewards which are determined by the action and previous history, and so this is actually a finite sum.
\begin{definition}[Discount Vector]
A {\it discount vector} $\du{k} \in [0,1]^\infty$ is a vector $\left[\dt{k}{1}, \dt{k}{2}, \dt{k}{3}, \cdots\right]$ satisfying
$\dt{k}{t} > 0$ for at least one $t \geq k$.
\end{definition}
The apparently superfluous superscript $k$ will be useful later when we allow the discount vector to change with time.
We do {\it not} insist that the discount vector be summable, $\sum_{t=k}^\infty \dt{k}{t} < \infty$.
\begin{definition}[Expected Values]
The expected discounted reward (or utility or value) when using policy $\pi$ starting in history $h_{<t}$ and discount vector $\du{k}$ is
\eq{
V^\pi_{\du{k}}(h_{<t}) := \Rw^\pi(h_{<t}) \cdot \du{k} := \sum_{i = 1}^\infty R^\pi(h_{<t})_i \dt{k}{i} =
\sum_{i = t}^\infty R^\pi(h_{<t})_i \dt{k}{i}.
}
\end{definition}
The sum can be taken to start from $t$ since $R^\pi(h_{<t})_i = 0$ for $i < t$. This means that the value of $\dt{k}{t}$ for
$t < k$ is unimportant, and never will be for any result in this paper.
As the scalar product is linear, a scaling of a discount vector has no affect on the ordering of the policies. Formally,
if $V^{\pi_1}_\du{k}(h_{<t}) \geq V^{\pi_2}_\du{k}(h_{<t})$ then $V^{\pi_1}_{\alpha\du{k}}(h_{<t}) \geq V^{\pi_2}_{\alpha\du{k}}(h_{<t})$ for all $\alpha > 0$.

\begin{definition}[Optimal Policy/Value]\label{defn_optimal_value}
In general, our agent will try to choose a policy $\pi^*_\du{k}$ to maximise $V^\pi_{\du{k}}(h_{<t})$. This is defined as follows.
\eq{
\pi^*_{\du{k}}(h_{<t}) &:= \argmax_{\pi} V^\pi_{\du{k}}(h_{<t}), &
\Rw^*_{\du{k}}(h_{<t}) &:= \Rw^{\pi^*_\du{k}}(h_{<t}), \\
V^*_{\du{k}}(h_{<t}) &:= V^{\pi^*_\du{k}}_{\du{k}}(h_{<t}).
}
\end{definition}
If multiple policies are optimal then $\pi^*_{\du{k}}$ is chosen
using some arbitrary rule.
Unfortunately, $\pi^*_\du{k}$ need not exist without one further assumption.
\begin{assumption}\label{assumption2}
For all $\pi$ and $k \geq 1$, $\lim_{t\to\infty} \sum_{h_{<t}} P(h_{<t} | \pi) V^\pi_{\du{k}}(h_{<t}) = 0$.
\end{assumption}
Assumption \ref{assumption2} appears somewhat arbitrary. We consider:
\begin{enumerate}
\item For summable $\du{k}$ the assumption is true for all environments. With the exception of hyperbolic
discounting, all frequently used discount vectors are summable.
\item For non-summable discount vectors $\du{k}$ the assumption implies a restriction on
the possible environments. In particular, they must return asymptotically lower rewards in expectation. This
restriction is necessary to guarantee the existence of the value function.
\end{enumerate}
From now on, including in theorem statements, we only consider environments/discount vectors satisfying
Assumptions \ref{assumption} and \ref{assumption2}.
The following theorem then guarantees the existence of $\pi^*_{\du{k}}$.

\begin{theorem}[Existence of Optimal Policy]\label{thm_existence}
$\pi^*_{\du{k}}$ exists for
any environment and discount vector $\du{k}$ satisfying Assumptions \ref{assumption} and \ref{assumption2}.
\end{theorem}
The proof of the existence theorem is in the appendix.

An agent can use a different discount vector $\du{k}$ for each time $k$. This motivates the following definition.
\begin{definition}[Discount Matrix]
A {\it discount matrix} $\d$ is a $\infty\times\infty$ matrix with discount vector $\du{k}$ for the $k$th column.
\end{definition}
It is important that we distinguish between a discount matrix $\d$ (written bold), a discount vector $\du{k}$
(bold and italics), and a particular value in a discount vector $\dt{k}{t}$ (just italics).
\begin{definition}[Sliding Discount Matrix]
A discount matrix $\d$ is {\it sliding} if $\dt{k}{k+t} = \dt{1}{t+1}$ for all $k, t \geq 1$.
\end{definition}

\begin{definition}[Mixed Policy]\label{defn_mixed}
The {\it mixed policy} is the policy where at each time step $t$, the agent acts according to the possibly different
policy $\pi^*_{\du{t}}$.
\eq{
\pi_\d(h_{<t}) &:= \pi^*_{\du{t}} (h_{<t}) &
\Rw_\d(h_{<t}) &:= \Rw^{\pi_\d}(h_{<t}).
}
\end{definition}
We do not denote the mixed policy by $\pi^*_\d$ as it is arguably not optimal as discussed in Section 5.
While non-unique optimal policies $\pi^*_\du{k}$ at least result in equal discounted utilities, this is {\it not}
the case for $\pi_\d$. All theorems are proved with respect to any choice $\pi_\d$.

\begin{definition}[Time Consistency]\label{defn_consistent}
A discount matrix $\d$ is {\it time consistent} if and only if for all environments
$\pi^*_{\du{k}}(h_{<t}) = \pi^*_{\du{j}}(h_{<t})$, for all $h_{<t}$ where $t \geq k, j$.
\end{definition}
This means that a time-consistent agent taking action $\pi^*_\du{t}(h_{<t})$ at each time $t$ will not change its
plans. On the other hand, a time-inconsistent agent may at time 1 intend to take action $a$ should it reach
history $h_{<t}$ ($\pi^*_\du{0}(h_{<t}) = a$). However upon reaching $h_{<t}$, it need not be true that $\pi^*_\du{t}(h_{<t}) = a$.

%%%%%%%%%%%%%%%%%%%%%%%%%%%%%%%%%%%%%%%%%%%%%%%%%%%%%%%%%%%%%%%
\section{Examples}
%%%%%%%%%%%%%%%%%%%%%%%%%%%%%%%%%%%%%%%%%%%%%%%%%%%%%%%%%%%%%%%

In this section we review a number of common discount matrices and give an example where a time-inconsistent discount
matrix causes very bad behavior.

%-------------------------------%
\subsubsect{Constant Horizon}
%-------------------------------%
Constant horizon discounting is where the agent only cares about the future up to $H$ time-steps away, defined by
$\dt{k}{t} = \ind{t - k < H}$.\footnote{$\ind{expr} = 1$ if $expr$ is true and $0$ otherwise.}
Shortly we will see that the constant horizon discount matrix can lead to very bad behavior in some environments.

%-------------------------------%
\subsubsect{Fixed Lifetime}
%-------------------------------%
Fixed lifetime discounting is where an agent knows it will not care about any rewards past time-step $m$, defined by
$\dt{k}{t} = \ind{t < m}$.
Unlike the constant horizon method, a fixed lifetime discount matrix is time-consistent. Unfortunately it requires
you to know the lifetime of the agent beforehand and also makes asymptotic analysis impossible.

%-------------------------------%
\subsubsect{Hyperbolic}
%-------------------------------%
$\dt{k}{t} = 1/(1 + \kappa (t - k))$.
The parameter $\kappa$ determines how farsighted the agent is with smaller values leading to more farsighted agents.
Hyperbolic discounting is often used in economics with some experimental studies explaining human time-inconsistent
behavior by suggesting that we discount hyperbolically \cite{Tha81}. The hyperbolic discount matrix is not summable, so may
be replaced by the following (similar to \cite{hut04}), which has similar properties for $\beta$ close to $1$.  \eq{
\dt{k}{t} = 1/(1 + \kappa (t - k))^\beta \text{ with } \beta > 1.
}

%-------------------------------%
\subsubsect{Geometric}
%-------------------------------%
$\dt{k}{t} = \gamma^{t}$ with $\gamma \in (0, 1)$.
Geometric discounting is the most commonly used discount matrix. Philosophically it can be justified by assuming
an agent will die (and not care about the future after death) with probability $1 - \gamma$ at each time-step.
Another justification for geometric discount is its analytic simplicity - it is summable and leads to time-consistent policies. It
also models fixed interest rates.

%-------------------------------%
\subsubsect{No Discounting}
%-------------------------------%
$\dt{k}{t} = 1, \text{ for all } k, t$.
\cite{HL07} and \cite{Leg08} point out that discounting future rewards via an explicit discount matrix is unnecessary
since the environment
can capture both temporal preferences for early (or late) consumption, as well as the risk associated with delaying
consumption.
Of course, this ``discount matrix'' is not summable, but
can be made to work by insisting that all environments satisfy Assumption \ref{assumption2}.
This approach is elegant in the
sense that it eliminates the need for a discount matrix, essentially admitting far more complex preferences
regarding inter-temporal
rewards than a discount matrix allows. On the other hand, a discount matrix gives the
``controller'' an explicit way to adjust the myopia of the agent.

\setlength{\intextsep}{0pt}
\begin{wrapfigure}[5]{r}{5.2cm}
\topsep=0.0cm
\begin{center}
\vspace{-0.1cm}
\tpic{
\node[state] (a) {\start};
\node[state] (b) [above of=a] {};
\node[state] (c) [right of=a] {};
\node[state] (d) [above of=c] {};
\node[state] (e) [right of=c] {};
\node[state] (f) [above of=e] {};
\node[state] (g) [right of=e] {};
\node[state] (h) [above of=g] {};
\node[state,draw=none] (x) [right of=h] {};
\node[state,draw=none] (y) [right of=g] {};

\path (a) edge node {$1/2$} (b)
      (a) edge node {$0$} (c)
      (c) edge node {$2/3$} (d)
      (c) edge node {$0$} (e)
      (e) edge node {$3/4$} (f)
      (e) edge node {$0$} (g)
      (g) edge node {$4/5$} (h)
      (b) edge node {$0$} (d)
      (d) edge node {$0$} (f)
      (f) edge node {$0$} (h)
      (g) edge node {} (y)
      (h) edge node {} (x)
      ;
      }
\end{center}
\end{wrapfigure}
To illustrate the potential consequences of time-inconsistent discount matrices we consider the policies of
several agents acting in the following environment.
Let agent A use a constant horizon discount matrix with $H = 2$ and agent B a geometric discount matrix with
some discount rate $\gamma$.

In the first time-step agent A prefers to move right with the intention of moving up in the second time-step for a reward of
$2/3$. However, once in second time-step, it will change its plan by moving right again. This continues indefinitely,
so agent A will always delay moving up and receives zero reward forever.

Agent B acts very differently. Let $\pi_t$ be the policy in which the agent moves right until time-step $t$, then up and
right indefinitely. $V^{\pi_t}_{\du{k}}(h_{<1}) = \gamma^t {(t + 1) \over (t + 2)}$. This value does not depend on $k$ and
so the agent will move right until $t = \argmax \left\{\gamma^t {(t + 1) \over {t + 2}} \right\} < \infty$ when it
will move up and receive a reward.

The actions of agent A are an example of the worst possible behavior arising from time-inconsistent discounting.
Nevertheless, agents with a constant horizon discount matrix are used in all kinds of problems. In particular, agents
in zero sum games where fixed depth mini-max searches are common. In practise, serious time-inconsistent behavior
for game-playing agents seems rare, presumably because most strategic games don't have a reward structure similar
to the example above.

%%%%%%%%%%%%%%%%%%%%%%%%%%%%%%%%%%%%%%%%%%%%%%%%%%%%%%%%%%%%%%%
\section{Theorems}
%%%%%%%%%%%%%%%%%%%%%%%%%%%%%%%%%%%%%%%%%%%%%%%%%%%%%%%%%%%%%%%

The main theorem of this paper is a complete characterisation of time consistent discount matrices.

\begin{theorem}[Characterisation] \label{thm_main}
Let $\d$ be a discount matrix, then the following are equivalent.
\begin{enumerate}
\item $\d$ is time-consistent (Definition \ref{defn_consistent})
\item For each $k$ there exists an $\alpha_k \in \R$ such that $\dt{k}{t} = \alpha_k \dt{1}{t}$ for all $t \geq k \in \N$.
\end{enumerate}
\end{theorem}
Recall that a discount matrix is sliding if $\dt{k}{t} = \dt{1}{t - k + 1}$.
Theorem \ref{thm_main} can be used to show that if a sliding discount matrix is used as in \cite{Str55} then the only
time-consistent discount matrix is geometric.
Let $\d$ be a time-consistent sliding discount matrix.
By Theorem \ref{thm_main} and the definition of sliding, $\alpha_1 \dt{1}{t + 1} = \dt{2}{t + 1} = \dt{1}{t}$.
Therefore ${1 \over \alpha_1} \dt{1}{2} = \dt{1}{1}$ and $\dt{1}{3} = {1 \over \alpha_1} \dt{1}{2} = \left(1 \over \alpha_1\right)^2 \dt{1}{1}$ and similarly, $\dt{1}{t} = \left(1\over \alpha_1\right)^{t-1} \dt{1}{1} \propto \gamma^t$ with
$\gamma = 1/\alpha_1$, which is geometric discounting. This is the analogue to the results of \cite{Str55} converted to our setting.

The theorem can also be used to construct time-consistent discount rates. Let $\du{1}$ be a discount vector, then the
discount matrix defined by
$\dt{k}{t} := \dt{1}{t}$ for all $t \geq k$ will always be time-consistent, for example, the {\it fixed lifetime} discount matrix with $\dt{k}{t} = 1$
if $t \leq H$ for some horizon $H$. Indeed, all time-consistent discount rates can be constructed in this way (up to scaling).

\begin{proof}[\proofof Theorem \ref{thm_main}]
$2 \Longrightarrow 1$: This direction follows easily from linearity of the scalar product.
\eqn{
\label{eqn-0-1} \pi^*_{\du{k}}(h_{<t}) &\equiv \argmax_\pi V^\pi_{\du{k}}(h_{<t}) \equiv \argmax_\pi \Rw^\pi(h_{<t}) \cdot \du{k}
= \argmax_\pi \Rw^\pi(h_{<t}) \cdot \alpha_k \du{1} \\
\nonumber &= \argmax_\pi \alpha_k \Rw^\pi(h_{<t}) \cdot \du{1}
= \argmax_\pi \Rw^\pi(h_{<t}) \cdot \du{1} \equiv \pi^*_{\du{1}}(h_{<t})
}
as required. The last equality of (\ref{eqn-0-1}) follows from the assumption that $\dt{k}{t} = \alpha_k \dt{1}{t}$ for all $t \geq k$ and
because $\Rw^\pi(h_{<t})_i = 0$ for all $i < t$.

$1 \Longrightarrow 2$: Let $\du{0}$ and $\du{k}$ be the discount vectors used at times $0$ and $k$ respectively.
Now let $k \leq t_1 < t_2 < \cdots$ and consider the deterministic environment below where the agent has a choice between earning
reward $r_1$ at time $t_1$ or $r_2$ at time $t_2$.
In this environment there are only two policies, $\pi_1$ and $\pi_2$, where
$\Rw^{\pi_1}(h_{<k}) = r_1 \v e_{t_1}$ and $\Rw^{\pi_2}(h_{<k}) = r_2 \v e_{t_{2}}$ with $\v e_i$ the infinite vector with all components zero except
the $i$th, which is $1$.
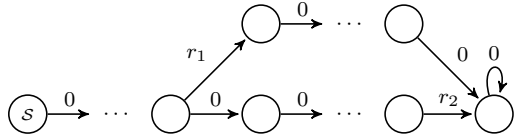
\begin{center}
\tpic{
\tikzstyle{every state}=[fill=none,draw=black,text=black]

\node[state] (X) {\start};
\node[state,draw=none] (A) [right of=X] {$\cdots$};
\node[state] (B) [right of=A, node distance=\shortnodedist] {};
\node[state,inner sep=0.025cm] (D) [right of=B] {};
\node[state,inner sep=0.025cm] (C) [above of=D] {};
\node[state,draw=none] (C1) [right of=C] {$\cdots$};
\node[state] (C2) [right of=C1, node distance=\shortnodedist] {};
\node[state,draw=none] (D1) [right of=D] {$\cdots$};
\node[state] (D2) [right of=D1, node distance=\shortnodedist] {};
\node[state,inner sep=0.025cm] (E) [right of=D2] {};

\path (B) edge node {$r_1$} (C)
      (B) edge node {$0$} (D)
      (C) edge node {$0$} (C1)
      (C2) edge node {$0$} (E)
      (D) edge node {$0$} (D1)
      (D2) edge node {$r_2$} (E)
      (X) edge node {$0$} (A)
      (E) edge [loop above] node {$0$} (E);
}
\end{center}
Since $\d$ is time-consistent, for all $r_1, r_2 \in \Re$ and $k \in \N$ we have:
\eqn{
\label{eqn1-1} \argmax_\pi V^{\pi}_{\du{1}}(h_{<k}) &\equiv \argmax_\pi \Rw^\pi(h_{<k}) \cdot \du{1} \\
&= \argmax_\pi \Rw^\pi(h_{<k}) \cdot \du{k} \equiv \argmax_\pi V^{\pi}_{\du{k}}(h_{<k}).
}
Now $V^{\pi_1}_\du{k} \geq V^{\pi_2}_\du{k}$ if and only if
$\du{k} \cdot \left[\Rw^{\pi_1}(h_{<k}) - \Rw^{\pi_2}(h_{<k})\right] =
[\dt{k}{t_1}, \dt{k}{t_2}] \cdot [r_1, -r_2] \geq 0$. Therefore we have that,
\eqn{
\label{eqn-main-eq1} [\dt{1}{t_1}, \dt{1}{t_2}] \cdot [r_1, -r_2] \geq 0 &\Leftrightarrow
[\dt{k}{t_1}, \dt{k}{t_2}] \cdot [r_1, -r_2] \geq 0.
}
Letting $\cos \theta_k$ be the cosine of the angle between $[\dt{k}{t_1}, \dt{k}{t_2}]$ and $[r_1, -r_2]$
then Equation (\ref{eqn-main-eq1}) becomes $\cos \theta_0 \geq 0 \Leftrightarrow \cos \theta_k \geq 0$.
Choosing $[r_1, -r_2] \propto [\dt{1}{t_2}, -\dt{1}{t_1}]$ implies that $\cos \theta_0 = 0$ and so $\cos \theta_k = 0$.
Therefore there exists $\alpha_k \in \R$ such that
\eqn{
\label{eqn-main-eq3} [\dt{k}{t_1}, \dt{k}{t_2}] = \alpha_k [\dt{1}{t_1}, \dt{1}{t_2}].
}
Let $k \leq t_1 < t_2 < t_3 < \cdots$ be a sequence for which $\dt{1}{t_i} > 0$. By the previous argument we have that,
$[\dt{k}{t_i}, \dt{k}{t_{i+1}}] = \alpha_k [\dt{1}{t_{i}}, \dt{1}{t_{i+1}}]$ and
$[\dt{k}{t_{i+1}}, \dt{k}{t_{i+2}}] = \tilde \alpha_k [\dt{1}{t_{i+1}}, \dt{1}{t_{i+2}}]$.
Therefore $\alpha_k = \tilde \alpha_k$, and by induction, $\dt{k}{t_i} = \alpha_k \dt{1}{t_i}$ for all $i$.
Now if $t \geq k$ and $\dt{1}{t} = 0$ then $\dt{k}{t} = 0$ by equation (\ref{eqn-main-eq3}). By symmetry,
$\dt{k}{t} = 0 \implies \dt{1}{t} = 0$. Therefore $\dt{k}{t} = \alpha_k \dt{1}{t}$ for all $t \geq k$ as required.
\proofend\end{proof}
In Section 3 we saw an example where time-inconsistency led to very bad behavior. The discount matrix causing
this was very time-inconsistent. Is it possible that an agent using a ``nearly''
time-consistent discount matrix can exhibit similar bad behavior? For example, could rounding errors when using
a geometric discount matrix seriously affect the agent's behavior? The following Theorem shows that this is not
possible.
First we require a measure of the cost of time-inconsistent behavior. The regret experienced by the agent at time
zero from following policy $\pi_\d$ rather than $\pi^*_\du{1}$ is $V^*_\du{1}(h_{<1}) - V^{\pi_\d}_\du{1}(h_{<1})$.
We also need a distance measure on the space of discount vectors.

\begin{definition}[Distance Measure]\label{defn_vector_distance}
Let $\du{k}, \du{j}$ be discount vectors then define a distance measure $D$ by
\eq{
D(\du{k}, \du{j}) := \sum_{i = \max\left\{k, j\right\}}^\infty |\dt{k}{i} - \dt{j}{i}|.
}
Note that this is almost the taxicab metric, but the sum is restricted to $i \geq \max\left\{k, j\right\}$.
\end{definition}

\begin{theorem}[Continuity]\label{thm_cont}
Suppose $\epsilon \geq 0$ and $\D{k}{j}:= D(\du{k}, \du{j})$ then
\eq{
V^*_\du{1}(h_{<1}) - V^{\pi_\d}_\du{1}(h_{<1}) \leq \epsilon + \D{1}{t} + \sum_{k=1}^{t-1} \D{k}{k+1}
}
with $t = \min\left\{t : \sum_{h_{<t}} P(h_{<t} | \pi^*_\du{1}) V^*_\du{1}(h_{<t}) \leq \epsilon\right\}$, which
for $\epsilon > 0$ is guaranteed to exist by Assumption \ref{assumption2}.
\end{theorem}
Theorem \ref{thm_cont} implies that the regret of the agent at time zero in its future time-inconsistent actions is
bounded by the sum of the differences between the discount vectors used at different times. If these differences
are small then the regret is also small. For example, it implies that small perturbations (such as rounding errors)
in a time-consistent discount matrix lead to minimal bad behavior.

The proof is omitted due to limitations in space. It relies on proving the result for finite horizon environments and showing that
this extends to the infinite case by using the horizon, $t$, after which the actions of the agent are no longer important.
The bound in Theorem \ref{thm_cont} is tight in the following sense.
\begin{theorem}\label{thm_lower_bound}
For $\delta > 0$ and $t \in \N$ and any sufficiently small $\epsilon > 0$ there exists an environment and discount matrix such that
\eq{
(t-2)(1 - \epsilon)\delta < V^*_\du{1}(h_{<1}) - V^{\pi_\d}_\du{1}(h_{<1}) &< (t+1)\delta  \\
&\equiv \D{1}{t} + \sum_{i=1}^{t-1} \D{i}{i+1}
}
where
$t = \min\left\{t : \sum_{h_{<t}} P(h_{<t} | \pi^*_\du{1}) V^*_\du{1}(h_{<t}) = 0 \right\} <\infty$ and
where $D(\du{k}, \du{j}) \equiv \D{k}{j} = \delta$ for all $k,j$.
\end{theorem}
Note that $t$ in the statement above is the same as that in the statement of Theorem \ref{thm_cont}.
Theorem \ref{thm_lower_bound} shows that there exists a discount matrix, environment and $\epsilon > 0$ where the regret
due to time-inconsistency is nearly equal to the bound given by Theorem \ref{thm_cont}.
\begin{proof}[\proofof Theorem \ref{thm_lower_bound}]
Define $\d$ by
\eq{
\dt{k}{i} = \begin{cases}
\delta & \text{if } k < i < t \\
0 & \text{otherwise}
\end{cases}
}
Observe that $D(\du{k}, \du{j}) = \delta$ for all $k < j < t$ since $\dt{j}{i} = \dt{k}{i}$ for all $i$ except $i = j$.
Now consider the environment below.
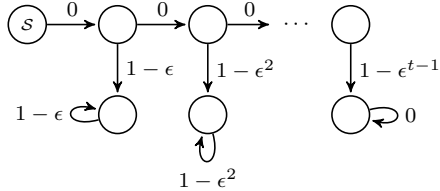
\begin{center}
\tpic{
\tikzstyle{every state}=[fill=none,draw=black,text=black]

\node[state] (a) {\start};
\node[state] (b) [right of=a] {};
\node[state] (c) [right of=b] {};
\node[state,draw=none] (h) [right of=c] {$\cdots$};
\node[state,node distance=\shortnodedist] (i) [right of=h] {};
\node[state] (e) [below of=b] {};
\node[state] (f) [below of=c] {};
\node[state] (j) [below of=i] {};

\path (a) edge node {$0$} (b)
      (b) edge node {$0$} (c)
      (c) edge node {$0$} (h)
      (b) edge node {$1 - \epsilon $} (e)
      (c) edge node {$1 - \epsilon^2$} (f)
      (i) edge node {$1 - \epsilon^{t-1}$} (j)
      (e) edge[loop left] node {$1 - \epsilon$} (e)
      (f) edge[loop below] node {$1 - \epsilon^2$} (f)
      (j) edge[loop right] node {$0$} (j)
      ;
}\end{center}
For sufficiently small $\epsilon$, the agent at time zero will plan to move right and then down leading to $\Rw^*_\du{1}(h_{<1}) = [0, 1 - \epsilon, 1 - \epsilon, \cdots]$ and $V^*_\du{1}(h_{<1}) = (t - 1) \delta (1 - \epsilon)$.

To compute $\Rw_\d$ note that $\dt{k}{k} = 0$ for all $k$. Therefore the agent in time-step $k$ doesn't care about the
next instantaneous reward, so prefers to move right with the intention of moving down in the next time-step when the rewards
are slightly better. This leads to
$\Rw_\d(h_{<1}) = [0, 0, \cdots, 1 - \epsilon^{t-1}, 0, 0, \cdots]$.
Therefore,
\eq{
V^*_\du{1}(h_{<1}) - V^{\pi_\d}_{\du{1}}(h_{<1}) &= (t - 1) \delta (1 - \epsilon) - (1 - \epsilon^{t-1}) \delta
\geq (t - 2)\delta (1 - \epsilon)
}
as required.
\proofend\end{proof}

%%%%%%%%%%%%%%%%%%%%%%%%%%%%%%%%%%%%%%%%%%%%%%%%%%%%%%%%%%%%%%%
\section{Game Theoretic Approach}
%%%%%%%%%%%%%%%%%%%%%%%%%%%%%%%%%%%%%%%%%%%%%%%%%%%%%%%%%%%%%%%

What should an agent do if it knows it is time inconsistent? One option is to treat its future selves as ``opponents''
in an extensive game. The game has one player per time-step who chooses the action for that time-step only. At the end of
the game the agent will have received a reward sequence $\v r \in \Re^\infty$. The utility given to the $k$th player is
then $\v r \cdot \du{k}$. So
each player in this game wishes to maximise the discounted reward with respect to a different
discounting vector.

\setlength{\intextsep}{0pt}
\begin{wrapfigure}[7]{r}{3.2cm}
\topsep=0.0cm
\begin{center}
\vspace{-0.1cm}
\tpic{
\tikzstyle{every state}=[fill=none,draw=black,text=black]

\node[state] (a) {\start};
\node[state] (b) [below of=a] {};
\node[state] (c) [right of=a] {};
\node[state] (d) [below of=c] {};
\node[state] (e) [right of=c] {};
\node[state] (f) [below of=e] {};

\path (a) edge node {$4$} (b)
      (a) edge node {$1$} (c)
      (c) edge node {$3$} (d)
      (c) edge node {$1$} (e)
      (f) edge[loop below] node {$0$} (f)
      (d) edge[loop below] node {$0$} (d)
      (b) edge[loop below] node {$0$} (b)
      (e) edge node {$3$} (f);
}
\end{center}
\end{wrapfigure}
For example, let $\du{1} = [2, 1, 2,0,0,\cdots]$ and $\du{2} = [*, 3, 1,0,0,\cdots]$ and consider the environment on the right.
Initially, the agent has two choices. It can either move down to guarantee a reward sequence
of $\v r = [4, 0, 0, \cdots]$ which
has utility of $\du{1} \cdot [4, 0, 0, \cdots] = 8$ or it can move right in which case it will receive a reward sequence of
either $\v r' = [1, 3, 0,0,\cdots]$ with utility $5$ or
$\v r'' = [1, 1, 3, 0,0, \cdots]$ with utility $9$. Which of these two reward sequences it receives is determined by the action taken
in the second time-step. However this
action is chosen to maximise utility with respect to discount sequence $\du{2}$ and $\du{2} \cdot \v r' > \du{1} \cdot \v r''$.
This means that if at time $1$ the agent chooses to move right, the final reward sequence will be $[1, 3, 0,0,\cdots]$ and the
final utility with respect to $\du{1}$ will be $5$. Therefore the rational thing to do in time-step 1 is to move down immediately
for a utility of $8$.

The technique above is known as backwards induction which is used to find sub-game perfect equilibria in finite
extensive games. A variant of Kuhn's theorem proves that backwards induction can be used to find such equilibria in
finite extensive games \cite{OR94}. For arbitrary extensive games (possibly infinite) a sub-game perfect
equilibrium need not exist, but we prove a theorem for our particular class of infinite games.

A sub-game perfect equilibrium policy is one the players could agree to play, and subsequently have no incentive to
renege on their agreement during play. It isn't always philosophically clear that a sub-game perfect equilibrium
policy {\it should} be played. For a deeper discussion, including a number of good examples, see \cite{OR94}.

\begin{definition}[Sub-game Perfect Equilibria]
A policy $\pi^*_\d$ is a sub-game perfect equilibrium policy if and only if for each $t$
$V^{\pi^*_\d}_\du{t}(h_{<t}) \geq V^{\tilde \pi}_\du{t}(h_{<t}), \text{ for all } h_{<t}$,
where $\tilde \pi$ is any policy satisfying $\tilde \pi(h_{<i}) = \pi^*_\d(h_{<i}) \forall h_{<i}$ where $i \neq t$.
\end{definition}

\begin{theorem}[Existence of Sub-game Perfect Equilibrium Policy]\label{thm_existence2}
For all environments and discount matrices $\d$ satisfying Assumptions \ref{assumption} and \ref{assumption2}
there exists at least one sub-game perfect equilibrium policy $\pi^{*}_{\d}$.
\end{theorem}
Many results in the literature of game theory almost prove this theorem. Our setting is more difficult than most
because we have countably many players (one for each time-step) and exogenous uncertainty. Fortunately, it is made easier by the
very particular conditions on the preferences of players for rewards
that occur late in the game (Assumption \ref{assumption2}). The closest related work appears to be that
of Drew Fudenberg in \cite{Fud83}, but our proof (see appendix) is very different.
The proof idea is to consider a sequence of environments identical to the original environment but with an increasing
bounded horizon after which reward
is zero. By Kuhn's Theorem \cite{OR94} a sub-game perfect equilibrium policy must exist in each of these
finite games. However the space of policies is compact (Lemma \ref{lem_compact}) and so this sequence of sub-game perfect equilibrium
policies contains a convergent sub-sequence converging to policy $\pi$. It is not then hard to show that
$\pi$ is a sub-game prefect equilibrium policy in the original environment.

\begin{proof}[\proofof Theorem \ref{thm_existence2}]
Add an action $a^{death}$ to $\A$ and $\mu$ such that if $a^{death}$ is taken at any time in $h_{<t}$ then $\mu$ returns zero reward.
Essentially, once in the agent takes action $a^{death}$, the agent receives zero reward forever.
Now if $\pi^*_\d$ is a sub-game perfect equilibrium policy in this modified environment then it is a sub-game perfect
equilibrium policy in the original one.

For each $t \in \N$ choose $\pi_t$ to be a sub-game perfect equilibrium policy in the further modified
environment obtained by
setting $r_i = 0$ if $i > t$. That is, the environment which gives zero reward always after time $t$.
We can assume without loss of generality that $\pi_t(h_{<k}) = a^{death}$ for all $k \geq t$.
Since $\Pi$ is compact, the sequence $\pi_1, \pi_2, \cdots$ has a convergent subsequence $\pi_{t_1}, \pi_{t_2}, \cdots$
converging to $\pi$ and satisfying
\begin{enumerate}
\item  $\pi_{t_i}(h_{<k}) = \pi(h_{<k}), \text{ for all } h_{<k}$ where $k \leq i$.
\item $\pi_{t_i}$ is a sub-game perfect equilibrium policy in the modified environment with reward $r_k = 0$ if $k > t_i$.
\item $\pi_{t_i}(h_{<t_i}) = a_{death}$.
\end{enumerate}
We write $\tilde V^{\pi_{t_i}}$ for the value function in the modified environment. It is now shown that $\pi$ is
a sub-game perfect equilibrium policy in the original environment. Fix a $t \in \N$ and
let $\tilde \pi$ be a policy with $\tilde \pi(h_{<k}) = \pi(h_{<k})$ for all $h_{<k}$ where $k \neq t$.
Now define policies $\tilde \pi_{t_i}$ by
\eq{
\tilde \pi_{t_i}(h_{<k}) = \begin{cases}
\tilde \pi(h_{<k}) & \text{if } k \leq i  \\
\pi_{t_i}(h_{<k}) & \text{otherwise }
\end{cases}
}
By point 1 above, $\tilde \pi_{t_i}(h_{<k}) = \pi_{t_i}(h_{<k})$ for all $h_{<k}$ where $k \neq t$.
Now for all $i > t$ we have
\eqn{
\label{eqn-ex1} V^\pi_\du{t}(h_{<t}) &\geq V^{\pi_{t_i}}_\du{t}(h_{<t}) - |V^\pi_\du{t}(h_{<t}) - V^{\pi_{t_i}}_\du{t}(h_{<t})| \\
\label{eqn-ex2} &\geq \tilde V^{\pi_{t_i}}_\du{t}(h_{<t}) - |V^\pi_\du{t}(h_{<t}) - V^{\pi_{t_i}}_\du{t}(h_{<t})| \\
\label{eqn-ex3} &\geq \tilde V^{\tilde \pi_{t_i}}_\du{t}(h_{<t}) - |V^\pi_\du{t}(h_{<t}) - V^{\pi_{t_i}}_\du{t}(h_{<t})| \\
\nonumber &\geq V^{\tilde \pi}_\du{t}(h_{<t}) - |V^\pi_\du{t}(h_{<t}) - V^{\pi_{t_i}}_\du{t}(h_{<t})| \\
\label{eqn-ex4}&\quad - |V^{\tilde \pi_{t_i}}_\du{t}(h_{<t}) - \tilde V^{\tilde \pi_{t_i}}_\du{t}(h_{<t})| -
|V^{\tilde \pi_{t_i}}_\du{t}(h_{<t}) - V^{\tilde \pi}_\du{t}(h_{<t})|
}
where (\ref{eqn-ex1}) follows from arithmetic. (\ref{eqn-ex2}) since $V \geq \tilde V$. (\ref{eqn-ex3}) since $\pi_{t_i}$ is
a sub-game perfect equilibrium policy. (\ref{eqn-ex4}) by arithmetic.
We now show that the absolute value terms in (\ref{eqn-ex4}) converge to zero.
Since $V^\pi(\cdot)$ is continuous in $\pi$ and $\lim_{i\to\infty} \pi_{t_i} = \pi$ and $\lim_{i\to\infty} \tilde \pi_{t_i} = \tilde \pi$, we obtain
$\lim_{i\to \infty} \left[|V^\pi_\du{t}(h_{<t}) - V^{\pi_{t_i}}_\du{t}(h_{<t})| + |V^{\tilde \pi_{t_i}}_\du{t}(h_{<t}) - V^{\tilde \pi}_\du{t}(h_{<t})|  \right] = 0$.
Now $\tilde \pi_{t_i}(h_{<k}) = a^{death}$ if $k \geq t_i$, so
$|V^{\tilde \pi_{t_i}}(h_{<t}) - \tilde V^{\tilde \pi_{t_i}}(h_{<t})| = 0$.
Therefore taking the limit as $i$ goes to infinity in (\ref{eqn-ex4}) shows that
$V^\pi_{\du{t}}(h_{<t}) \geq V^{\tilde \pi}_\du{t}(h_{<t})$
as required.
\proofend\end{proof}
In general, $\pi^*_\d$ need not be unique, and different sub-game equilibrium policies can lead to different utilities.
This is a normal, but unfortunate,
problem with the sub-game equilibrium solution concept. The policy is unique if for all players the value of any two arbitrary  policies is
different. Also, if
$\forall k (V^{\pi_1}_\du{k} = V^{\pi_2}_\du{k} \implies \forall j V^{\pi_1}_\du{j} = V^{\pi_2}_\du{j})$ is true
then the non-unique sub-game
equilibrium policies have the same values for all agents. Unfortunately, neither of these conditions is
necessarily satisfied in our setup. The problem of how players might choose a sub-game perfect equilibrium
policy appears surprisingly understudied. We feel it provides another reason to avoid the
situation altogether by using time-consistent discount matrices.
The following example illustrates the problem of non-unique sub-game equilibrium policies.
\begin{example}
Consider the example in Section 3 with an agent using a constant horizon discount matrix with $H=2$. There are
exactly two sub-game perfect equilibrium policies, $\pi_1$ and $\pi_2$ defined by,
\eq{
\pi_1(h_{<t}) &= \begin{cases}
up& \text{if } t \text{ is odd} \\
right& \text{otherwise}
\end{cases} &
\pi_2(h_{<t}) &= \begin{cases}
up& \text{if } t \text{ is even} \\
right& \text{otherwise}
\end{cases}
}
Note that the reward sequences (and values) generated by $\pi_1$ and $\pi_2$ are different with $\Rw^{\pi_1}(h_{<1}) =
[1/2, 0, 0, \cdots]$ and $\Rw^{\pi_2}(h_{<1}) = [0, 2/3,0,0,\cdots]$.
If the players choose to play a sub-game perfect equilibrium policy then the first player can choose between $\pi_1$
and $\pi_2$ since they have the first move. In that case it would be best to follow $\pi_2$ by moving right as it
has a greater return for the agent at time $0$ than $\pi_1$.
\end{example}
For time-consistent discount matrices we have the following proposition.
\begin{proposition}
If $\d$ is time-consistent then $V^*_\du{k} =V^{\pi_\d}_\du{k} = V^{\pi^*_\d}_\du{k}$  for all $k$ and choices
of $\pi^*_\du{k}$ and $\pi_\d$ and $\pi^*_\d$.
\end{proposition}
Is it possible that backwards induction is simply expected discounted reward
maximisation in another form? The following theorem shows this is not the case and that
sub-game perfect equilibrium policies are a rich and interesting class worthy of further study in
this (and more general) settings.
\begin{theorem}\label{thm_non_eq}
$\exists \d\text{ such that }\pi^*_\d \neq \pi^*_\tdu{0}, \text{ for all } \tdu{0}$.
\end{theorem}

The result is proven using a simple counter-example. The idea is to construct a stochastic environment where the first
action leads the agent to one of two sub-environments, each with probability half. These environments are identical
to the example at the start of this section, but one of them has the reward $1$ (rather than $3$) for the history $right,down$.
It is then easily shown that $\pi^*_\du{}$ is not the result of an expectimax expression because it behaves differently in each
sub-environment, while any expectimax search (irrespective of discounting) will behave the same in each.

%%%%%%%%%%%%%%%%%%%%%%%%%%%%%%%%%%%%%%%%%%%%%%%%%%%%%%%%%%%%%%%
\section{Discussion}
%%%%%%%%%%%%%%%%%%%%%%%%%%%%%%%%%%%%%%%%%%%%%%%%%%%%%%%%%%%%%%%

%-------------------------------%
\subsubsect{Summary}
%-------------------------------%
Theorem \ref{thm_main} gives a characterisation of time-(in)consistent discount matrices and shows that all
time-consistent discount matrices follow the simple form of $\dt{k}{t} = \dt{1}{t}$.
Theorem \ref{thm_cont} shows
that using a discount matrix that is nearly time-consistent produces mixed policies with low regret. This is useful
for a few reasons, including showing that small perturbations, such as rounding errors, in a discount matrix cannot cause
major time-inconsistency problems. It also shows that ``cutting off'' time-consistent discount matrices after some fixed
depth - which makes the agent potentially time-inconsistent - doesn't affect the policies too much, provided the depth is large enough.
When a discount matrix is very time-inconsistent then taking a game theoretic approach may dramatically decrease
the regret in the change of policy over time.

Some comments on the policies $\pi^*_\du{k}$ (policy maximising expected $\du{k}$-discounted reward), $\pi_\d$ (mixed
policy using $\pi^*_\du{k}$ at each time-step $t$)
and $\pi^*_\d$ (sub-game perfect equilibrium policy).
\begin{enumerate}
\item A time-consistent agent should play policy $\pi^*_\du{k} = \pi_\d$ for any $k$. In this case, every optimal policy
$\pi^*_\du{k}$ is also a sub-game perfect equilibrium policy.
\item $\pi_\d$ will be played by an agent that believes it is time-consistent, but may not be. This can lead to very bad
behavior as shown in Section 3.
\item An agent may play $\pi^*_\d$ if it knows it is time-inconsistent, and also knows exactly how
(I.e, it knows $\du{k}$ for all $k$ at every time-step). This policy is arguably rational, but comes with its own
problems, especially non-uniqueness as discussed.
\end{enumerate}

%-------------------------------%
\subsubsect{Assumptions}
%-------------------------------%
We made a number of assumptions about which we make some brief comments.
\begin{enumerate}
\item Assumption \ref{assumption}, which states that $\A$ and $\O$ are finite, guarantees the existence of an optimal policy. Removing the
assumption would force us to use $\epsilon$-optimal policies, which shouldn't be a problem for the theorems to go through
with an additive $\epsilon$ slop term in some cases.
\item Assumption \ref{assumption2} only affects non-summable discount vectors. Without it,
 even $\epsilon$-optimal policies need not exist and all the machinery will break down.
\item The use of discrete time greatly reduced the complexity of the analysis. Given a sufficiently general
model, the set of continuous environments should contain all discrete environments. For this reason the proof of
Theorem \ref{thm_main} should go through essentially unmodified. The same may not be true for Theorems \ref{thm_cont} and
\ref{thm_existence2}. The former may be fixable with substantial effort (and perhaps should be true intuitively). The latter
has been partially addressed, with a positive result in \cite{Gol80,BM73,Pol68,Str55}.
\end{enumerate}

%%%%%%%%%%%%%%%%%%%%%%%%%%%%%%%%%%%%%%%%%%%%%%%%%%%%%%%%%%%%%%%
%%                 Bibliography                              %%
%%%%%%%%%%%%%%%%%%%%%%%%%%%%%%%%%%%%%%%%%%%%%%%%%%%%%%%%%%%%%%%

\begin{small}

\end{small}

\appendix
%%%%%%%%%%%%%%%%%%%%%%%%%%%%%%%%%%%%%%%%%%%%%%%%%%%%%%%%%%%%%%%
\section{Technical Proofs}
%%%%%%%%%%%%%%%%%%%%%%%%%%%%%%%%%%%%%%%%%%%%%%%%%%%%%%%%%%%%%%%

Before the proof of Theorem \ref{thm_existence} we require a definition and two lemmas.
\begin{definition}\label{defn_policy_space}
Let $\Pi = \A^{\S}$ be the set of all policies and define a metric $D$ on $\Pi$ by
$T(\pi_1, \pi_2) := \min_{t\in\N} \left\{t :\exists h_{<t} \text{ s.t } \pi_1(h_{<t})\neq \pi_2(h_{<t})  \right\} \text{ or } \infty \text { if } \pi_1 = \pi_2$ and
$D(\pi_1, \pi_2) := \exp(-T(\pi_1, \pi_2))$.
\end{definition}
$T$ is the time-step at which $\pi_1$ and $\pi_2$ first differ.
Now augment $\Pi$ with the topology induced by the metric $\d$.

\begin{lemma}\label{lem_compact}
$\Pi$ is compact.
\end{lemma}
\begin{proof}
We proceed by showing $\Pi$ is totally bounded and complete. Let $\epsilon = \exp(-t)$ and define an equivalence relation by
$\pi \sim \pi' \text{ if and only if } T(\pi_1, \pi_2) \geq t$.
If $\pi \sim \pi'$ then $D(\pi, \pi') \leq \epsilon$.
Note that $\Pi /_\sim$ is finite. Now choose a representative from each class to create a finite set $\bar \Pi$.
Now $\bigcup_{\pi \in \bar \Pi} B_\epsilon(\pi) = \Pi$, where $B_\epsilon(\pi)$ is the ball
of radius $\epsilon$ about $\pi$. Therefore $\Pi$ is totally bounded.

Next, to show $\Pi$ is complete.
Let $\pi_1, \pi_2, \cdots$ be a Cauchy sequence with $D(\pi_i, \pi_{i+j}) < \exp(-i)$ for all $j > 0$.
Therefore $\pi_i(h_{<k}) = \pi_{i+j}(h_{<k}) \forall h_{<k}$ with $k \leq i$, by the definition of $D$.
Now define
$\pi$ by $\pi(h_{<t}) := \pi_t(h_{<t})$ and note that $\pi_i(h_{<j}) = \pi(h_{<j}) \forall j \leq i$ since $\pi_i(h_{<k}) = \pi_k(h_{<k}) \equiv \pi(h_{<k})$ for $k \leq i$. Therefore
$\lim_{i \to\infty} \pi_i = \pi$ and so $\Pi$ is complete.
Finally, $\Pi$ is compact by the Heine-Borel theorem.
\proofend\end{proof}
\begin{lemma}\label{lem_cont}
When viewed as a function from $\Pi$ to $\R$, $V^\pi_\du{k}(\cdot)$ is continuous. (given Assumption \ref{assumption2})
\end{lemma}
\begin{proof}
Suppose $D(\pi_1,  \pi_2) < \exp(-t)$ then $\pi_1$ and $\pi_2$ are identical on all histories up to length $t$.
Therefore
\eqn{
\nonumber |V^{\pi_1}_{\du{k}}(h_{<k}) - V^{\pi_2}_{\du{k}}(h_{<k})| &\leq \du{k}\cdot \left[\Rw^{\pi_1}(h_{<k}) + \Rw^{\pi_2}(h_{<k})\right] \\
\label{eqn10-1} &=\sum_{i=k}^\infty \dt{k}{i} \left(R^{\pi_1}(h_{<k})_i + R^{\pi_2}_i(h_{<k})_i\right).
}
Since $\pi_1$ and $\pi_2$ are identical up to time $t$, (\ref{eqn10-1}) becomes
\eqn{
\nonumber &\sum_{i=t}^\infty \dt{k}{i} \left(R^{\pi_1}(h_{<k})_i + R^{\pi_2}_i(h_{<k})_i\right) = \\
\label{eqn11-1} &\qquad  \sum_{h_{<t}} \left[P(h_{<t} | h_{<k}, \pi_1) V^{\pi_1}_{\du{k}}(h_{<t}) + P(h_{<t} | h_{<k}, \pi_2)V^{\pi_2}_{\du{k}}(h_{<t})| \right]
}
where (\ref{eqn11-1}) follows from the definition of the reward and value functions.
By Assumption \ref{assumption2}, $\lim_{t\to \infty} \sum_{h_{<t}} P(h_{<t}|h_{<k}, \pi_i) V^{\pi_i}_{\du{k}}(h_{<t}) = 0$ for $i \in \left\{1, 2 \right\}$ and so,
$V$ is continuous.
\proofend\end{proof}

\begin{proof}[\proofof Theorem \ref{thm_existence}]
Let $\Pi$ be the space of all policies with the metric of Definition \ref{defn_policy_space}.
By Lemmas \ref{lem_compact}/\ref{lem_cont} $\Pi$ is compact and $V$ is continuous.
Therefore $\argmax_\pi V^\pi_{\du{k}}(h_{<1})$ exists by the extreme value theorem.
\proofend\end{proof}

%%%%%%%%%%%%%%%%%%%%%%%%%%%%%%%%%%%%%%%%%%%%%%%%%%%%%%%%%%%%%%%
\section{Table of Notation}
%%%%%%%%%%%%%%%%%%%%%%%%%%%%%%%%%%%%%%%%%%%%%%%%%%%%%%%%%%%%%%%

{\small
\begin{tabular}{|l | p{12.5cm}|}
\hline
{\bf Symbol} & {\bf Description} \\
$\d$ & Discount Matrix \\
$\du{k}$ & Discount Vector $k$ \\
$\dt{k}{t}$ & The $t$th component of discount vector $\du{k}$ (at time $k$ reward $r_t$ is discounted by $\dt{k}{t}$) \\
$k$, $t$ & Indices. $k$ usually referring to a discount vector used at fixed time $k$, $t$ usually a time index for states. \\
$i$ & Summing index \\
$\epsilon, \delta$ & Small real numbers greater than zero \\
$\pi, \pi', \pi_i$ & Policies \\
$\Pi$ & The space of all policies \\
$\A, \S, \O, \Re$ & Action, state, reward and observation spaces \\
$R(s, a)$ & The reward given when taking action $a$ in state $s$ \\
$P(s' | s, a)$ & The probability of transitioning to state $s'$ from state $s$ having taken action $a$ \\
$\N, \R$ & The natural and real numbers respectively \\
$\S_t$ & The set of all states reachable at time-step $t$ \\
$\S_{<t}$ & The set of all states reachable up to time-step $t$ \\
$B_\epsilon(\cdot)$ & A ball of radius $\epsilon$ \\
$\Rw^\pi(h_{<t})$ & The expected reward sequence when following $\pi$ from state $h_{<t}$ \\
$\pi^*_\du{k}$ & The optimal policy when using discount vector $\du{k}$ \\
$\pi_\d$ & The mixed policy using discount matrix $\d$ \\
$\pi^*_\d$ & The sub-game equilibrium policy using discount matrix $\d$ \\
$\Rw^*_\du{k}(h_{<t})$ & The expected reward sequence when following the optimal policy $\pi^*_\du{k}$ \\
$V^*_\du{k}(h_{<t})$ & The value of the optimal policy $\pi^*_\du{k}$ \\
$\gamma$ & Discount rate for geometric discounting \\
$\alpha_k$ & A real valued scaling factor on a discount vector \\
$\kappa$ & Discount rate for hyperbolic discounting \\
$h$ & Horizon for constant depth discounting \\
$m$ & Lifespan for fixed lifetime discounting \\
$s, h_{<t}, h_{<t}'$ & States in a Markov decision process \\
$D(\pi_1, \pi_2)$ & The distance between policies $\pi_1$ and $\pi_2$ using the metric of Definition \ref{defn_policy_space} \\
$D(\du{k}, \du{j})$ & The distance measure between discount vectors $\du{k}$ and $\du{j}$ as defined by Definition \ref{defn_vector_distance} \\
\hline
\end{tabular}
}

\end{document}